%% file: main.tex
\documentclass{article}
\usepackage[preprint]{corl_2026} 

\usepackage{microtype}
\usepackage{graphicx}
\usepackage{subcaption}
\usepackage{booktabs} 

\usepackage{hyperref}
\usepackage{enumitem}

\usepackage{amsmath}
\usepackage{amssymb}
\usepackage{mathtools}
\usepackage{amsthm}

\usepackage[capitalize,noabbrev]{cleveref}

\theoremstyle{plain}
\newtheorem{theorem}{Theorem}[section]
\newtheorem{proposition}[theorem]{Proposition}

\theoremstyle{definition}
\newtheorem{definition}[theorem]{Definition}

\theoremstyle{remark}

\usepackage{float}
\input{preamble}

\usepackage[textsize=tiny]{todonotes}

\title{\textcolor{myblue}{PriGo}: Test-Time \textcolor{myblue}{Pri}mitive \textcolor{myblue}{G}uidance to Diffusion and Flow Policies for Adaptive R\textcolor{myblue}{o}botic Manipulation}
\author{Zezeng Li$^{1}$, ~ Enda Xiang$^{2}$, ~Thuy Tran$^{1}$, ~Di Huang$^{2}$, ~Momath Thiam$^{1}$, ~Liming Chen$^1$\\ 	
\normalsize{$^1$École Centrale de Lyon, France}\qquad
\normalsize{$^2$Beihang University, China}
}

\begin{document}

\maketitle

\input{sections/0_abstract}    
\input{sections/1_introduction}

\input{sections/2_related_works}

\input{sections/3_method}
\input{sections/4_experiments}
\input{sections/5_conclusion}

\clearpage

\acknowledgments{This work was in part supported by the French Research Agency ANR, l’Agence Nationale de Recherche, through the projects Chiron (ANR-20-IADJ-0001-01), Aristotle (ANR-21-FAI1-0009-01),  Astérix (ANR-23-EDIA-0002), and  DEMETER project (ANR-25-HTCE-0002), the French national investment priority program through the PSPC FAIR WASTE project, and the Franco-Chinese Research Center for Carbon Neutrality through the Artemis project.}

\bibliography{main}

\input{sections/A_appendix}

\end{document}

%% file: preamble.tex
\usepackage{amsthm}
\usepackage{tabularx}
\usepackage{multirow}
\usepackage{multicol}
\usepackage{colortbl}
\usepackage{bm} 
\usepackage{longtable}
\usepackage{xcolor}

\definecolor{myblue}{HTML}{5364cc}
\definecolor{ImportantColor}{RGB}{204, 255, 229}

\newcommand{\SE}{\mathrm{SE}}

\newcommand{\ac}{\mathbf{A}}
\newcommand{\ob}{\mathbf{O}}
\newcommand{\as}{\mathbf{a}}

%% file: sections/0_abstract.tex
\begin{abstract}
Imitation learning has enabled remarkable progress in robotic manipulation, especially with diffusion and flow-based policies that generate complex visuomotor behaviors directly from demonstrations. Yet, despite their strong performance, these policies often fail to generalize across tasks and environments. A key reason is that existing policies tend to imitate superficial action correlations rather than the underlying intent. Inspired by the compositional structure of human behaviors, we propose \textcolor{myblue}{\textbf{PriGo}}, a primitive-guided test-time adaptive framework for robust robotic manipulation. PriGo introduces PANet, a lightweight primitive prediction module that infers primitive distributions directly from observations. We further propose a differentiable primitive guidance mechanism that refines generated actions during inference, steering trajectories toward semantically consistent behaviors. Unlike prior primitive-conditioned approaches, PriGo operates entirely at test time and can be seamlessly integrated into pretrained diffusion and flow policies without retraining. Extensive experiments on LIBERO, CALVIN, SIMPLER, and real-world robotic tasks demonstrate that PriGo consistently improves robustness, long-horizon execution, and generalization ability across both diffusion and flow-based policies.
Codes are available on \href{https://github.com/OTCVCG/PriGo}{\textcolor{myblue}{\textbf{PriGo}}}.
\end{abstract}



%% file: sections/1_introduction.tex
\section{Introduction}
Imitation learning (IL) has become a dominant paradigm for robotic manipulation, enabling policies to acquire complex visuomotor behaviors directly from expert demonstrations without manually designed reward functions. Recent diffusion and flow-based policies have further improved action generation quality by modeling multi-modal action distributions and leveraging large-scale vision-language representations. Despite these advances, existing policies often remain brittle under distribution shifts, including variations in object, lighting, distractors, and unseen task~\citep{sapkota2025vision,li2025robotic}.

A key limitation is that current policies primarily learn correlations between observations and low-level actions, rather than the underlying structural intent of manipulation behaviors. Consequently, generated actions may locally resemble expert demonstrations while failing to preserve the intended motion structure required for successful task completion. As shown in Fig.~\ref{fig:teaser}.a, when opening a door, the policy incorrectly pulls before rotating the key, or generates unstable grasping motions under environmental perturbations. Such failures become particularly severe in long-horizon manipulation, where small local errors accumulate across multiple subtasks. To improve robustness and generalization, an important direction is to incorporate structured action priors into policy generation. Human manipulation behaviors naturally exhibit compositional structure, where complex tasks are composed of a small number of reusable motion primitives such as grasping, pushing, pulling, and rotation. Rather than treating actions as unconstrained continuous trajectories, these primitives provide a compact structural prior that constrains policies toward semantically consistent behaviors.

Recent works~\cite{hiranaka2023primitive,gao2024prime} have sought to disentangle primitives into discrete types with associated parameters, capturing both categorical abstractions and low-level action prediction. However, they limit the primitive set to basic actions, which restricts their applicability to more complex tasks. Furthermore,

\begin{figure}
\vspace{-1.5mm}
    \begin{minipage}{0.45\linewidth}
     these methods treat primitives as additional input rather than enforcing them as mandatory constraints, resulting in generated actions that do not adhere to the required primitives.
     
     To this end, we introduce PriGo, a primitive-guided test-time adaptation framework for robotic manipulation policies. Instead of retraining existing policies, PriGo introduces structured primitive guidance during inference in a plug-and-play manner. Specifically, we first train a lightweight primitive classifier, PANet, which predicts primitive distributions directly from visual observations and language instructions. We adopt a compact 
    \end{minipage}
    \hfill
    \vspace{-1.5mm}
    \begin{minipage}{0.53\linewidth}
    \centering
    \includegraphics[width=1.0\linewidth]{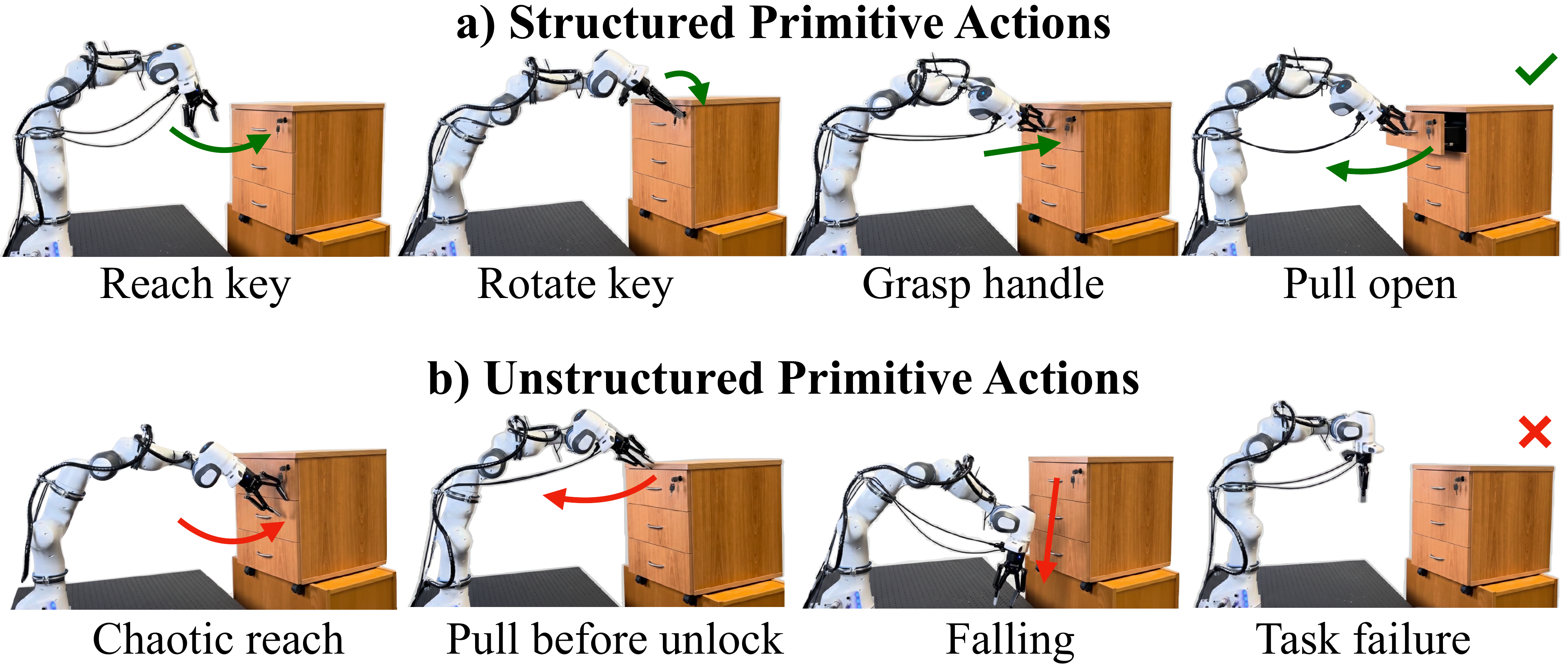}
    \vspace{-3mm}
    \caption{Comparison of a) our policy with PriGo, which produces structured primitive actions, and b) vanilla policies without PriGo, which lack structured primitive action outputs.}
    \label{fig:teaser}
    \end{minipage}
\end{figure}

\noindent primitive taxonomy consisting of eight manipulation primitives, providing a balance between structural expressiveness and cross-task generalization.

We then incorporate differentiable primitive guidance into both diffusion and flow policies. During inference, actions are refined using gradients derived from primitive consistency objectives, encouraging trajectories to align with the predicted manipulation structure. Unlike prior guided diffusion approaches that mainly rely on task rewards or language supervision, PriGo introduces structured action-level guidance tailored for robotic manipulation. Our contributions are summarized as follows:
\begin{itemize}[itemsep=0pt, leftmargin=12pt]
\vspace{-2mm}
\item We propose PriGo, a plug-and-play primitive-guided test-time adaptation framework for robotic manipulation policies.
\item We introduce a differentiable primitive guidance mechanism that constrains diffusion and flow policy generation toward structured manipulation behaviors without retraining.
\item We design a lightweight primitive classification framework with stable cross-task primitive representations and probabilistic multi-primitive guidance.
\item We conduct comprehensive experiments and provide thorough analysis, demonstrating the superiority of \textbf{PriGo} in both simulation and real robot tasks, outperforming state-of-the-art methods.
\end{itemize}

%% file: sections/2_related_works.tex
\section{Related Works}
\vspace{-2mm}
\paragraph{Diffusion and Flow Policies for Manipulation.} 
Diffusion policies~(DPs)~\cite{chi2023diffusion,black2023zero,xian2023chaineddiffuser,wangadamanip,ze20243d,yangequibot,3d_diffuser_actor,wang2024equivariant,he2025learning,xiang2026graspldp,team2024octo,wen2024diffusion,li2024cogact,zhou2025chatvla} have recently emerged as an effective framework for visuomotor policy learning by modeling action generation as iterative denoising processes. Existing approaches can be broadly divided into two categories. The first is \textbf{lightweight DPs} ~\cite{chi2023diffusion,black2023zero,xian2023chaineddiffuser,wangadamanip,ze20243d,yangequibot,3d_diffuser_actor,wang2024equivariant,he2025learning,xiang2026graspldp}, which have evolved from the original 2D diffusion policy to 3D variants and, more recently, to equivariant formulations that better exploit symmetries. The second category comprises \textbf{generalist DPs}~\cite{team2024octo,wen2024diffusion,li2024cogact,zhou2025chatvla,liu2024rdt,liu2025hybridvla}, which combine diffusion with large-scale vision–language models~(VLM) to enable general-purpose robotic behaviors across diverse tasks.

Parallel to diffusion-based approaches, flow matching has also been adopted for robotic policy learning due to its stable optimization and efficient inference. Existing approaches can be broadly categorized into two groups: (1) lightweight flow-based policies, such as FlowMS~\cite{rouxel2024flow}, FMP~\cite{zhang2024affordance}, and ActionFlow~\cite{funk2024actionflow}, which focus on stable trajectory generation and spatial reasoning; and (2) VLM-integrated flow policies and VLA systems, including $\pi_0$~\cite{black2410pi0}, HiRobot~\cite{shi2025hi}, GraspVLA~\cite{deng2025graspvla}, and SmolVLA~\cite{shukor2025smolvla}, which combine flow matching with large-scale vision-language representations to improve generalization, sim-to-real transfer, and deployment efficiency.

Despite their success, existing diffusion and flow policies remain sensitive to distribution shifts at test time. Several recent works attempt to address this issue through guided sampling or test-time adaptation. MetaDiffuser~\cite{ni2023metadiffuser} employs classifier guidance, using pretrained models to steer the diffusion process, while LDuS~\cite{kim2024llm} leverages large language models (LLMs) to bias sampling toward goal-consistent actions. More recently, ADPro~\cite{li2025adpro} proposes task-aware initialization and constrains updates along both task and spherical manifolds, yielding a more generalizable policy. However, these approaches primarily focus on task-level objectives or reward guidance, and do not explicitly enforce structured manipulation behaviors during action generation.

\begin{figure}[htbp!]
\centering
\includegraphics[width=1.0\textwidth]{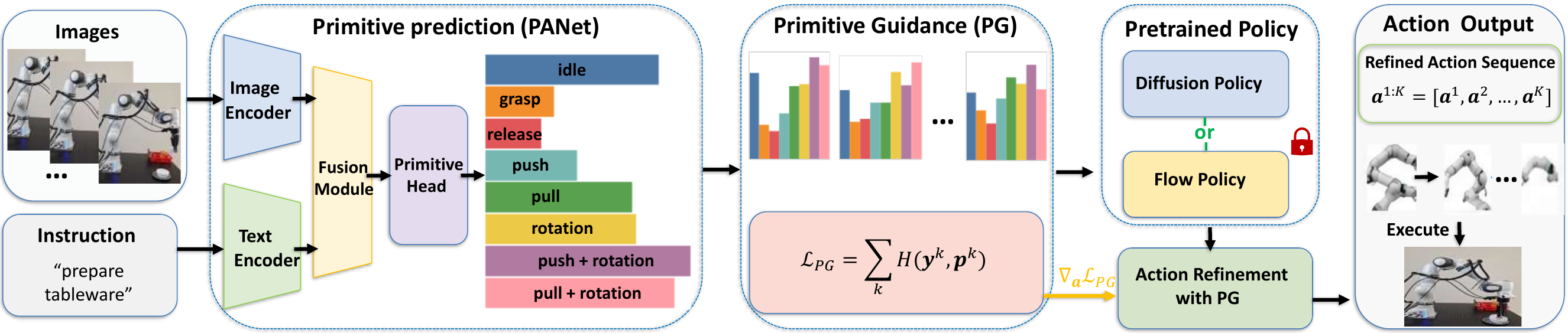}
\caption{Overview of \textcolor{myblue}{\textbf{PriGo}}, a primitive-guided framework for robotic manipulation policies. Given visual observations and language instructions, the primitive classifier PANet predicts probabilistic manipulation primitives. During inference, generated actions from pretrained diffusion or flow policies are converted into soft primitive distributions and aligned with the predicted primitives through differentiable primitive guidance. The resulting gradient-based refinement steers action toward semantically consistent trajectories without retraining the backbone policy.} 
\label{fig:framework}
\end{figure}

\paragraph{Primitive-based Policies.} Primitive-based control has long been studied for improving policy compositionality and generalization. Early approaches \cite{peng2019mcp,merel2019neural,knaust2021guided,dalal2021accelerating} employed compositional networks to combine primitives into complex behaviors, where the composite policy's action distribution was a weighted sum over a set of \textbf{latent feature primitives}. However, their expressiveness was constrained by the simplicity of Gaussian primitives. Other researchers leverage LLMs to decompose task instructions into \textbf{prompt primitives}, using chain-of-thought reasoning to guide policy execution \cite{zhang2023bootstrap,garcia2025towards,pan2025omnimanip}. While effective in some cases, they rely heavily on the reliability of LLM-based decomposition. 

Several recent works further explore \textbf{extensible network primitives} that expand through incrementally added policy modules, often implemented with lightweight adapters such as LoRA \cite{sun2025learning,liu2025skill,zhou2025chatvla}, enabling continual learning across diverse tasks \cite{wan2024lotus,guo2025srsa}. Nevertheless, it raises challenges in storage efficiency and in managing experts' allocation as the library scales. More recent works \cite{hiranaka2023primitive,gao2024prime} sought to disentangle primitives into discrete types with associated parameters. Yet, they typically restrict the primitive set to only a handful of basic actions, limiting applicability to more complex tasks. 

Different from prior works, our work focuses on test-time primitive guidance rather than primitive-conditioned policy learning. PriGo does not require retraining the policy and instead introduces differentiable primitive constraints during inference. Furthermore, our formulation naturally supports probabilistic multi-primitive guidance, enabling smooth transitions between manipulation primitives.

%% file: sections/3_method.tex
\vspace{-2mm}
\section{Method}
\vspace{-3mm}
As illustrated in Fig.~\ref{fig:framework}, PriGo introduces primitive-guided test-time adaptation for robotic manipulation policies. Our framework consists of two components: a primitive classifier PANet that predicts primitive distributions from observations, and a differentiable primitive guidance mechanism that constrains action generation during inference. PriGo can be seamlessly integrated into pretrained diffusion and flow policies without retraining.

\textbf{Problem Statement.}
Given a set of expert demonstrations containing action-observation pairs, denoted as $\mathcal{D} = \{(\ob^k,\ac^k )\}_{k=1}^N$, our objective is to fit a manipulation policy that maps the observation $\ob$ to the action $\ac$. Here, $\ob$ consists of RGB images and a language instruction, and action $\as:= ( \mathbf{x}, \mathbf{r}, w)$ is a 7D pose, defined by translation $\mathbf{x}=(x,y,z)$, rotation $\mathbf{r}=(\text{roll},\text{pitch},\text{yaw})$, and the gripper width $w\in[0,1]$. 
To avoid ambiguity, we use the superscript $k$ of $\as^k$ to denote the frame index, while the subscript $t$ of noisy action $\as_t$ denotes the time steps during the denoising process.

\subsection{Primitive Representation and Classification}

Instead of directly constraining continuous actions, PriGo introduces a compact set of manipulation primitives as structured action priors $\mathcal{Y}=\{$``idle'', ``grasp'', ``release'', ``push'', ``pull'', ``rotation'', ``push+rotation'', ``pull+rotation''$\}$. 
We adopt this compact primitive taxonomy to balance structural expressiveness and cross-task generalization. Empirically, overly fine-grained primitive decompositions introduce long-tail imbalance and semantic ambiguity, while the proposed primitive set yields stable primitive distributions across diverse tasks.

\begin{definition}
\label{def:HardClassification}
\textbf{(Primitive Action Hard Classification)}.
Given the thresholds $\tau_{\text{trans}}, \tau_{\text{rot}}, \tau_w > 0$, the hard classification $\Phi:\as^k\to y$ for an action $\as:= ( \mathbf{x}, \mathbf{r}, w)$ is formulated from the translational magnitudes $\|\mathbf{x}^k\|$, the rotational magnitudes $\|\mathbf{r}^k\|$ and the gripper width $w$ as:
\begin{equation}\label{eq:HardClassify}
\Phi(\as^k) =
\begin{cases}
0 & \text{if } \|\mathbf{x}^k\| < \tau_{\text{trans}}, \;\|\mathbf{r}^k\| < \tau_{\text{rot}}, |w^k - w^{k-1}| < \tau_w \quad (\text{idle})\\
1 & \text{if } w^{k-1} - w^k \geq \tau_w \quad (\text{grasp}) \\
2 & \text{if } w^{k} - w^{k-1} \geq \tau_w \quad (\text{release}) \\
3 & \text{if } \|\mathbf{x}^k\| \geq \tau_{\text{trans}}, \;\|\mathbf{r}^k\| < \tau_{\text{rot}}, |w^k-w^{k-1}| < \tau_w, \mathbf{x}^k_1 > 0 \quad (\text{push}) \\
4 & \text{if } \|\mathbf{x}^k\| \geq \tau_{\text{trans}}, \;\|\mathbf{r}^k\| < \tau_{\text{rot}}, |w^k-w^{k-1}| < \tau_w, \mathbf{x}^k_1 < 0 \quad (\text{pull}) \\
5 & \text{if } \|\mathbf{x}^k\| < \tau_{\text{trans}}, \;\|\mathbf{r}^k\| \geq \tau_{\text{rot}}, |w^k-w^{k-1}| < \tau_w,  \quad (\text{rotation})\\
6 & \text{if } \|\mathbf{x}^k\| \geq \tau_{\text{trans}}, \;\|\mathbf{r}^k\| \geq \tau_{\text{rot}}, \mathbf{x}^k_1 > 0, |w^k-w^{k-1}| < \tau_w \quad (\text{push+rotation})\\
7 & \text{if } \|\mathbf{x}^k\| \geq \tau_{\text{trans}}, \;\|\mathbf{r}^k\| \geq \tau_{\text{rot}}, \mathbf{x}^k_1 < 0, |w^k-w^{k-1}| < \tau_w \quad (\text{pull+rotation})
\end{cases}
\end{equation}
\end{definition}

Definition~\ref{def:HardClassification} is formulated in a robot-centric action frame to ensure simplicity and stable cross-task consistency, while being complete for 7-DOF manipulation control. Through primitive compositions, this representation supports a wide range of robotic manipulation tasks. Eq.~\eqref{eq:HardClassify} enables automatic generation of primitive labels, thus allowing us to learn a primitive classifier without manual annotation. We then train a classifier PANet: $\hat{\mathbf{y}}^k=\text{PANet}(\ob^k)$, which predicts primitive distributions from observations. PANet adopts DINOv2~\cite{oquab2024dinov2} as the visual encoder and T5~\cite{2020t5} as the text encoder, followed by a multimodal fusion module and MLP classification head.

Importantly, PANet predicts probabilistic primitive distributions rather than hard labels. This probabilistic formulation improves robustness under ambiguous transitions between primitives and naturally supports multi-primitive guidance during long-horizon manipulation.


\subsection{Primitive Guidance}\label{sec:guidance}
To enable differentiable guidance while preserving local consistency with the corresponding hard primitive assignments, we define a soft primitive distribution induced by the generated actions.
\begin{definition}
\label{def:softClassification}
\textbf{(Soft Classification)}. Given the differentiable score functions:
$
\mathbf{f}^k := \big(
-(\|\mathbf{x}^k\| + \|\mathbf{r}^k\| + |w^k - w^{k-1}|), \;\;
w^{k-1} - w^k, \;\;
w^k - w^{k-1}, \;\;
\mathbf{x}^k_1, \;\;
-\mathbf{x}^k_1, \;\;
(\|\mathbf{r}^k\| - \|\mathbf{x}^k\|), \;\;
\mathbf{x}^k_1 + \|\mathbf{r}^k\|, \;\;
-\mathbf{x}^k_1 + \|\mathbf{r}^k\|
\big).
$
The soft classification, i.e., the probability distribution over primitive categories, is obtained using a temperature-scaled softmax, i.e., $\mathbf{p}^k = \text{softmax}\!\left(\mathbf{f}(\mathbf{a}^k)/\tau\right)$,
where $\tau > 0$ is a temperature.
\end{definition}

Given predicted primitive distributions $\hat{\mathbf{y}}^k$ and generated actions $\mathbf{a}^k$, the primitive guidance is:
\begin{equation}
\mathcal{L}_{PG}(\as^k,\ob^k)
=
-\frac{1}{B}
\sum_{k=1}^{B}
H(\hat{\mathbf{y}}^k,\mathbf{p}^k)= -\frac{1}{B} \sum_{i=1}^B H(\text{PANet}(\ob^k),\mathbf{p}^k)\,,
\end{equation}
where $H(\cdot,\cdot)$ denotes cross-entropy between predicted primitive distributions and action-induced primitive distributions, $B$ is the batch size. Unlike hard primitive constraints or single-label guidance, this probabilistic formulation enables smooth transitions between primitives and improves robustness during long-horizon execution. In particular, multiple primitives may simultaneously receive non-negligible probabilities during transitional stages, such as switching from \texttt{push} to \texttt{rotation}. This soft guidance reduces abrupt action corrections and improves trajectory consistency.

\subsection{Diffusion Policy with Primitives Guidance}\label{sec:dp_guidance}
Diffusion policy formulates the manipulation process as a Markov process and leverages DDPM~\cite{ho2020denoising} to model the manipulation policy with multimodal observation $\ob^k$. Given observation $\ob^k$, it performs a sequence of $T$ denoising steps starting from a random action $\as_{T}^k\sim \mathcal{N}(\mathbf{0}, \mathbf{I})$ to generate an action $\as_0^k$ with iteration formula $\as_{t-1}^k=\alpha_t(\as_t^k - \gamma_t \varepsilon_\theta(\ob^k, \as_t^k, t) + \epsilon)$,
where $\epsilon\sim \mathcal{N}(\mathbf{0}, \sigma_t^2\mathbf{I})$. The coefficients $\alpha_t, \gamma_t, \sigma_t$ are pre-defined noise schedules of the denoising step $t$. 

To improve the transferability of diffusion policies, we incorporate the primitive guidance mechanism introduced in \Cref{sec:guidance} into the denoising process. By leveraging test-time observations to refine generated actions, the proposed approach steers denoising updates toward primitive-consistent trajectories, resulting in more robust and semantically coherent action generation under distribution shifts. This yields the following guided diffusion policy:
\begin{equation}\label{eq:primdf}
\begin{split}
&\Tilde{\as}^k_{t-1} = \alpha_t(\as_t^k - \gamma_t \varepsilon_\theta(\ob^k, \as_t^k, t) + \epsilon) \,, \\
&\hat{\as}^k_{0} =(\Tilde{\as}^k_{t-1} -\sqrt{1-\bar{\alpha}_t}\varepsilon_\theta(\ob^k, \as_t^k, t))/\sqrt{\bar{\alpha}_t}\,,\\
&\as^k_{t-1} =\Tilde{\as}^k_{t-1} - \eta \nabla_{\as} \mathcal{L}_{\text{PG}}(\hat{\as}_0^k, \ob^k) \,,
\end{split}
\end{equation}
Here, $\hat{\as}^k_{0}$ denotes the estimated noise-free action obtained through the pre-defined inverse diffusion schedule. By incorporating the primitive guidance term $\nabla_{\as}\mathcal{L}_{\text{PG}}(\cdot,\cdot)$, Eq.~\eqref{eq:primdf} constrains the update direction of the inverse diffusion process, ensuring consistency with the intended primitive actions.

\subsection{Flow Policy with Primitives Guidance}\label{sec:fp_guidance}
Flow policies leverage flow matching to learn a continuous transformation from noisy to expert actions. Given an action $\as$, the flow path was defined as $\as_t = \alpha_t\as + b_t \mathbf{z}, t \in [0,1]$, where $\alpha_t$ and $b_t$ are pre-defined schedules, $\mathbf{z} \sim p(\mathbf{z})$ is an initial noise action. A commonly used schedule is $\alpha_t=t$ and $b_t=1-t$, then the velocity field is $\mathbf{v}(\as_t, t) = \frac{\mathrm{d} \as_t}{\mathrm{d}t}=\as-\mathbf{z}$. The flow policy, parameterized by $\theta$, predicts velocity field $\mathbf{v}_\theta(\as_t,\ob, t)$. After training, it generate $\as_1$ from noise $\as_0=\mathbf{z}$ by using forward Euler rule $\as_{t+\delta} = \as_t + \delta\mathbf{v}_\theta(\as_t,\ob,t)$, where $\delta$ is the integration step size.

Similarly, we integrate test-time primitive guidance in \Cref{sec:guidance} into the inference phase of flow policies.
Specifically, we leverage the primitives derived from test-time observation to constrain the continuous actions generated, thereby yielding actions that are consistent with the target primitives. From the pre-defined flow schedule, we can deduce the following guided flow policy:
\begin{equation}\label{eq:primfp}
\begin{split}
    &\tilde{\as}^k_{t+\delta} = \as^k_t + \delta\mathbf{v}_\theta(\as^k_t,\ob^k,t),\\
    &\hat{\as}^k_1 = \as^k_t+(1-t)\mathbf{v}_\theta(\as^k_t,\ob^k,t),\\
    &\as^k_{t+\delta} = \tilde{\as}^k_{t+\delta}-\eta\nabla_{\as}\mathcal{L}_\text{PG}(\hat{\as}^k_1,\ob^k).
\end{split}
\end{equation}
This is equivalent to defining a guided velocity field $\mathbf{v}_{\theta}^{\text{guided}} =
\mathbf{v}_{\theta} - \frac{\eta}{\delta} \nabla_{\as} \mathcal{L}_{\text{PG}}(\hat{\as}_1, \ob)$.

Our method imposes stronger constraints during the action refinement process for both diffusion-based and flow-based policies, thereby generating actions that conform to the desired structured primitives. Compared to the original policy models, the guided policies produce actions whose deviation from the target actions is bounded, a property that we further elaborate on in the supplementary material.

%% file: sections/4_experiments.tex
\section{Experiments}\label{sec:experiments}
\vspace{-2mm}
We design experiments to answer the following questions: \textbf{i)} How accurate and robust is PANet for primitive prediction? \textbf{ii)} Can the proposed primitive guidance improve the manipulation performance of pretrained policies without retraining? \textbf{iii)} Does PriGo generalize effectively to real-world robotic manipulation tasks? \textbf{iv)} How does each core component contribute to the overall performance? 
We first evaluate the accuracy and robustness of PANet in \Cref{sec:accrob}. Next, we demonstrate the performance improvements of PriGo on simulated benchmarks in \Cref{sec:simulated} and real-world robotic tasks in \Cref{sec:Real-Robot}. Finally, we conduct ablation studies in \Cref{sec:ablation}.

\begin{figure}[htbp]
  \centering
  \begin{minipage}[c]{0.52\textwidth}
    \centering
    \scriptsize 
    \resizebox{\linewidth}{!}{
        \renewcommand{\arraystretch}{1.1} 
        \setlength{\tabcolsep}{0.6mm}     
        \begin{tabular}{@{}lcccccccc@{}}
          \toprule
          Type & Idle & Grasp & Rel. & Push & Pull & Rot. & Push+Rot. & Pull+Rot. \\ \midrule
          Idle  & \textbf{90} & 6  & 2   & 1   & 1  & 0  & 0  & 0 \\
          Grasp & 7 & \textbf{96} & 1  & 0  & 1  & 0  & 0  & 0  \\
          Rel. & 1 & 2 & \textbf{96} & 1 & 0 & 0 & 0 & 0 \\
          Push & 0  & 0  & 0 & \textbf{94} & 2 & 1 & 3 & 0 \\
          Pull & 0  & 0  & 0 & 1 & \textbf{95} & 1 & 0 & 3 \\
          Rot. & 0  & 0  & 0 & 1 & 2 & \textbf{93} & 2 & 2 \\
          Push+Rot. & 0 & 0 & 0& 5  & 0& 4  & \textbf{91} & 0 \\
          Pull+Rot.& 0 & 0  & 0  & 0 & 6  & 4  & 0 & \textbf{90} \\ \bottomrule
        \end{tabular}
    }
  \end{minipage}%
  \hfill 
  \begin{minipage}[c]{0.45\textwidth}
    \centering
    \includegraphics[width=\linewidth]{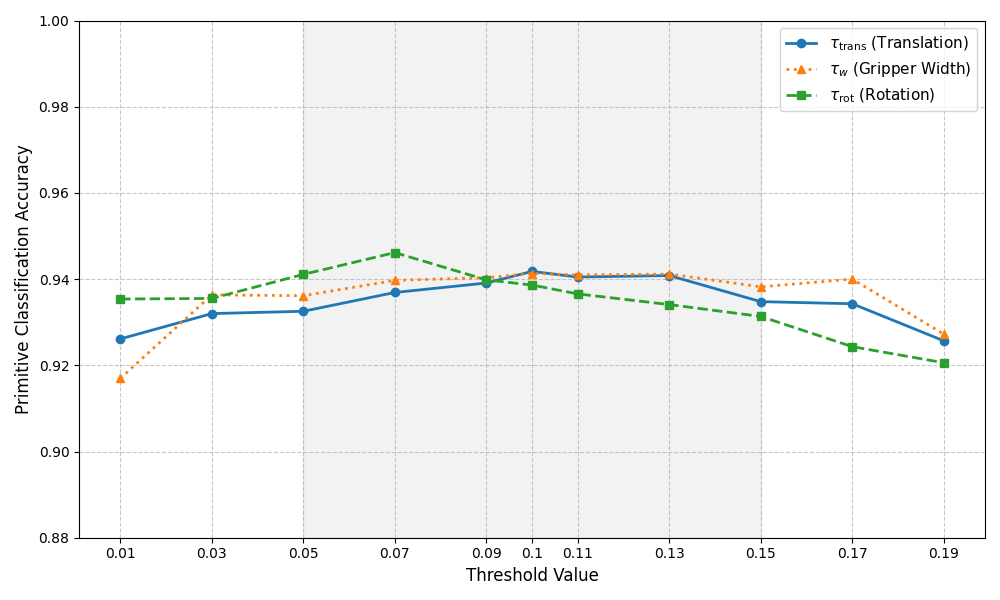}
  \end{minipage}
  \vspace{-3mm}
  \caption{Confusion matrix of PANet primitive classification (left) and sensitivity analysis of auto-labeling thresholds (right). Rel. and Rot. denote release and rotation, respectively.}
  \label{fig:confusion_thresholds}
\end{figure}

\subsection{Analysis of PANet Robustness}\label{sec:accrob}
To evaluate the reliability of PANet, we report its primitive classification accuracy and confusion matrix in the left subfigure of Fig.~\ref{fig:confusion_thresholds}. PANet achieves 94.0\% accuracy on LIBERO. Most errors occur between semantically adjacent primitives, such as \texttt{push+rotation} and \texttt{push} or \texttt{rotation}, indicating that misclassifications largely remain within similar local action manifolds.

We further analyze the sensitivity of PriGo to the thresholds $\tau_{\text{trans}}$, $\tau_{\text{rot}}$, and $\tau_w$. As shown in the right subfigure of Fig.~\ref{fig:confusion_thresholds}, performance remains stable within a broad range of threshold values, demonstrating that PriGo is not overly sensitive to specific hyperparameter value choices.

\begin{table}[htb!]
\vspace{-3mm}
\small
\caption{Success rate on LIBERO with four settings (spatial, object, goal, long-horizon). Comparison between with~(w/ PG) and without~(w/o PG) primitive guidance PriGo. Here $\pi_0$+PG and SmolVLA+PG are PriGo-FP. DP+PG and CogACT+PG are PriGo-DP.}
\renewcommand{\arraystretch}{0.8}
\renewcommand{\tabcolsep}{0.88mm}
\centering
\begin{tabular}{l|cc|cc|cc|cc|cc}
\toprule
\multirow{2}{*}{Methods}& \multicolumn{2}{c|}{Spatial}  &  \multicolumn{2}{c|}{Object}&   \multicolumn{2}{c|}{Goal} &	\multicolumn{2}{c|}{Long} & \multicolumn{2}{c}{\cellcolor{ImportantColor}Avg.}\\
&w/o PG & w/ PG & w/o PG & w/ PG &w/o PG & w/ PG &w/o PG & w/ PG & \cellcolor{ImportantColor}w/o PG & \cellcolor{ImportantColor}w/ PG \\
\midrule
$\pi_0$~\cite{black2410pi0} & 89.2 & 91.0 & 84.6 & 87.8 & 94.4 & 95.2 & 75.2 & 78.8 &\cellcolor{ImportantColor} 85.6 &\cellcolor{ImportantColor} 88.2 \\
SmolVLA~\cite{shukor2025smolvla} & 92.0 & 93.6 & 93.4 & 94.4 & 91.4 & 93.2 & 75.8 & 79.8 &\cellcolor{ImportantColor} 88.2   &\cellcolor{ImportantColor} 90.3\\
\midrule
DP~\cite{chi2023diffusion} & 76.2 & 80.6 & 87.4 &  90.8 & 70.4 & 72.6 & 54.2 & 58.4 &\cellcolor{ImportantColor}72.1 &\cellcolor{ImportantColor}75.6 \\
CogACT~\cite{li2024cogact} & 91.1 & 95.6 & 92.0 & 97.3 & 91.5 & 95.8 & 76.4 & 80.4   &\cellcolor{ImportantColor}87.8  &\cellcolor{ImportantColor}92.3 \\
\bottomrule
\end{tabular}
\label{tab:LIBERO}
\end{table}

\subsection{Simulated Evaluation}\label{sec:simulated}
\vspace{-2mm}
\textbf{Evaluation on LIBERO.} LIBERO~\cite{liu2023libero} is a large-scale benchmark encompassing over 130 tasks and categorized into five distinct suites. To verify the effectiveness of the proposed primitive guidance, we integrate it into a suite of pretrained policies, including the flow policies $\pi_0$~\cite{black2410pi0}, SmolVLA~\cite{shukor2025smolvla}, diffusion policies DP~\cite{chi2023diffusion}, and CogACT~\cite{li2024cogact}. This integration yields two policy variants, namely PriGo-DP and PriGo-FP. For $\pi_0$ and SmolVLA, we adopt the pre-trained models and hyperparameter configurations provided by LeRobot\footnote{https://github.com/huggingface/lerobot}. As for DP and CogAct, no official pre-trained checkpoints on LIBERO are available, we thus fine-tuned these two models on LIBERO. 

\Cref{tab:LIBERO} presents the evaluation results of four pretrained policies, with and without the integration of our primitive guidance module. As shown, the proposed PriGo consistently improves performance by 3–5 points across flow-based and diffusion-based policies, regardless of whether the backbone is a large-scale model (e.g., $\pi_0$ and CogAct) or a lightweight one (DP). This consistent improvement stems from PriGo’s direct intervention: by enforcing generated actions to conform to the primitives required by the task, PriGo produces more structured actions, which leads to higher success rates.

\textbf{Evaluation on SIMPLER.} SIMPLER benchmark~\cite{li2025evaluating} provides two evaluation settings: \textit{Visual Matching}, which closely mirrors real-world scenarios by minimizing the gap between simulated and physical environments, and \textit{Variant Aggregations}, which introduces variations in background, lighting, distractors, and table textures. We compare our PriGo-DP with five other VLA models, with the corresponding results summarized in \Cref{tab:google_robot}. Notably, all baseline results presented in this table are from CogACT. 
As shown, PriGo-DP outperforms all existing baseline methods by a clear margin. Specifically, relative to CogACT, our approach delivers an average improvement of 5 percent under the Visual Matching setting and 7 points under the Variant Aggregation setting. These results demonstrate that PriGo can effectively enhance the test-time performance of the VLA model CogACT in a plug-and-play manner, without the need for additional retraining.

\begin{table*}[!htbp]
\centering
\vspace{-4mm}
\small
\renewcommand{\arraystretch}{0.8}
\renewcommand{\tabcolsep}{0.99mm}
\caption{Comparison of our approach PriGo-DP~(CogACT with PriGo) with existing VLA models on the Google robot across four tasks in two SIMPLER settings. }
\vspace{-2mm}
\begin{tabular}{l|l|cccc|c}
\multirow{2}{*}{Google Robot} &
  \multirow{2}{*}{Method} &
  Pick &
  Move &
  Open/Close & Open Top Drawer & \multirow{2}{*}{\emph{Average}} \\ 
  & & Coke Can & Near & Drawer & and Place Apple & \\
  \midrule
\multirow{6}{*}{\begin{tabular}[l]{@{}l@{}}SIMPLER \\ (Visual Matching) \end{tabular}}
             & RT-2-X~\cite{vuong2023open}   & 78.7   &  77.9  & 25.0  &  ~~3.7 & \cellcolor{ImportantColor}46.3  \\ 
             & Octo-Base~\cite{team2024octo}      & 17.0     & ~~4.2    & 22.7     &  ~~0.0    & \cellcolor{ImportantColor}11.0    \\ 
             & OpenVLA~\cite{kim2025openvla}      & 18.0     & 56.3     & 63.0     & ~~0.0     &\cellcolor{ImportantColor} 34.3   \\
             & CogACT~\cite{li2024cogact} & 91.3     & 85.0     & 71.8    &  50.9   & \cellcolor{ImportantColor}74.8  \\
             & \textbf{PriGo-DP} & \textbf{93.2}     & \textbf{89.3}     & \textbf{77.5}    &  \textbf{59.7}   &\cellcolor{ImportantColor} \textbf{79.9}  \\
             \midrule
\multirow{6}{*}{\begin{tabular}[l]{@{}l@{}}SIMPLER \\ (Variant Aggregation) \end{tabular}}
             & RT-2-X~\cite{vuong2023open}  & 82.3  &  79.2 & \underline{35.3} & 20.6 &\cellcolor{ImportantColor} 54.4 \\ 
             & Octo-Base~\cite{team2024octo}     & ~~0.6     & ~~3.1    & ~~1.1     &  ~~0.0    & \cellcolor{ImportantColor}~~1.2    \\ 
             & OpenVLA~\cite{kim2025openvla}  & 60.8     & 67.7     & 28.8     &  ~~0.0   & \cellcolor{ImportantColor} 39.3    \\
             & CogACT~\cite{li2024cogact}  & 89.6     & 80.8     & 28.3     & 46.6     & \cellcolor{ImportantColor}61.3      \\  
             & \textbf{PriGo-DP}  & \textbf{92.5}     & \textbf{83.6}     & \textbf{44.1}     & \textbf{55.0}     & \cellcolor{ImportantColor}\textbf{68.8}      \\   
\bottomrule
\end{tabular}
\label{tab:google_robot}
\end{table*}

\textbf{Evaluation on CALVIN.}
CALVIN~\cite{mees2022calvin} is a benchmark centered on a Franka Panda robot interacting within richly parameterized tabletop environments. It has four environments that differ in desk textures and object placements. The benchmark defines 34 tasks with sequential execution of language-conditioned instructions, where each instruction comprises five steps. We integrate the proposed primitive guidance into the pretrained diffusion policy 3D Diffuser Actor (3DDA)~\cite{3d_diffuser_actor}, and compare its performance with that of several SoTA methods on the CALVIN, including DTP~\cite{hou2024diffusion}, DP3~\cite{ze20243d}, 3DDA, and ADPro~\cite{li2025adpro}. The results are reported in \Cref{tab:calvin}.

\begin{table}[!h]
\vspace{-1mm}
\begin{minipage}[th!]{0.45\textwidth}
As shown in \Cref{tab:calvin}, our PriGo-DP (3DDA with PriGo), outperforms other diffusion policies by a notable margin. Specifically, PriGo-DP increases the average length compared to the vanilla diffusion policy 3DDA by 0.8, enhancing its performance on zero-shot long-horizon tasks. These results demonstrate that PriGo can effectively boost test-time generalization, even for non-VLA models.
\end{minipage}
\hfill
\vspace{-1mm}
\begin{minipage}[th!]{0.53\textwidth}
\captionsetup{type=table}
\caption{Zero-shot (Train A, B, C $\rightarrow$Test D) long-horizon evaluation on CALVIN. Avg. Len denotes the average length of correctly predicted trajectories.}
\footnotesize
\renewcommand{\arraystretch}{0.6}
\renewcommand{\tabcolsep}{0.82mm}
\centering
\begin{tabular}{@{}l|ccccc|c}
\toprule
\multirow{2}{*}{Method}  & \multicolumn{5}{c}{Task completed in a row}& \multirow{2}{*}{Avg. Len $\uparrow$}  \\
& 1 & 2 & 3 & 4 & 5&\\
    \midrule
    DTP & 94.5 & 82.5 & 72.8 & 61.3 & 50.0 & \cellcolor{ImportantColor}3.61 \\
    DP3  & 53.9 & 44.7 & 38.0 & 34.3 & 29.0 & \cellcolor{ImportantColor}2.00  \\
    3DDA & 93.8 & 80.3 & 66.2 & 53.3 & 41.2 & \cellcolor{ImportantColor}3.35 \\
    ADPro & 94.7 & 83.0 & 73.6 & 61.4 & 51.1 & \cellcolor{ImportantColor}3.64 \\
    \textbf{PriGo-DP}  & 98.1 & 92.1  & 87.6 & 80.4 & 59.9 & \cellcolor{ImportantColor}4.18  \\
\bottomrule
\end{tabular}
\label{tab:calvin}
\end{minipage}
\end{table}
%


\begin{figure*}[!htbp] 
    \centering
    \begin{subfigure}{\textwidth}
        \centering
        \includegraphics[width=0.9\textwidth,height=0.10\textwidth]{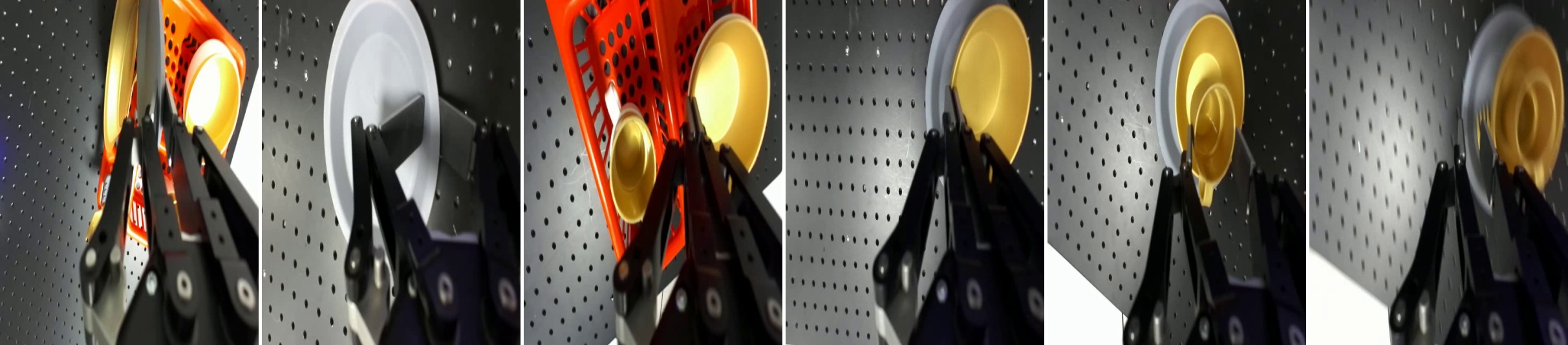}
        \caption{Rollout for the \textit{``Prepare a set of tableware} task''.}
    \end{subfigure}
    \begin{subfigure}{\textwidth}
        \centering
        \includegraphics[width=0.9\textwidth,height=0.10\textwidth]{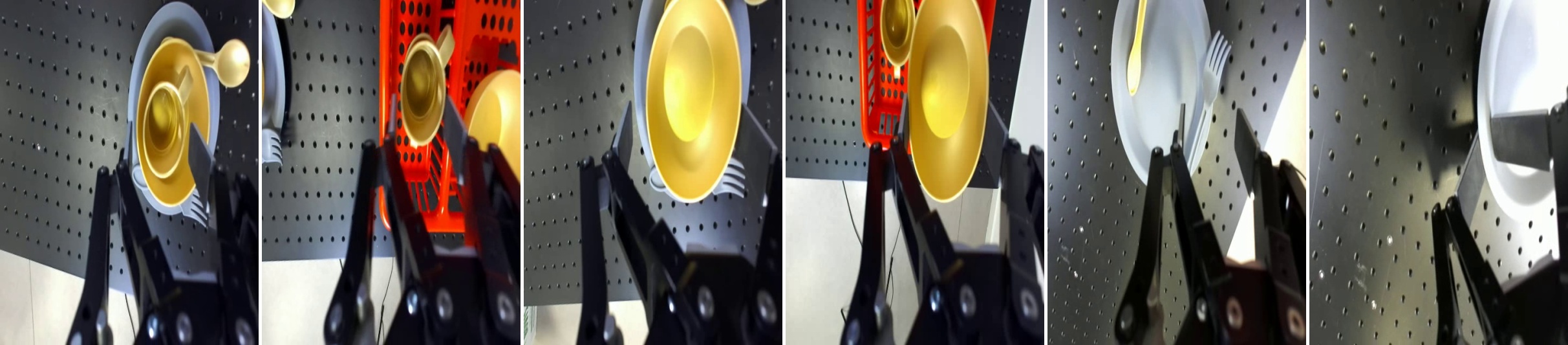}
        \caption{Rollout for the \textit{``Bring back a set of tableware to the basket''} task.}
    \end{subfigure}
    \begin{subfigure}{\textwidth}
        \centering
        \includegraphics[width=0.9\textwidth,height=0.10\textwidth]{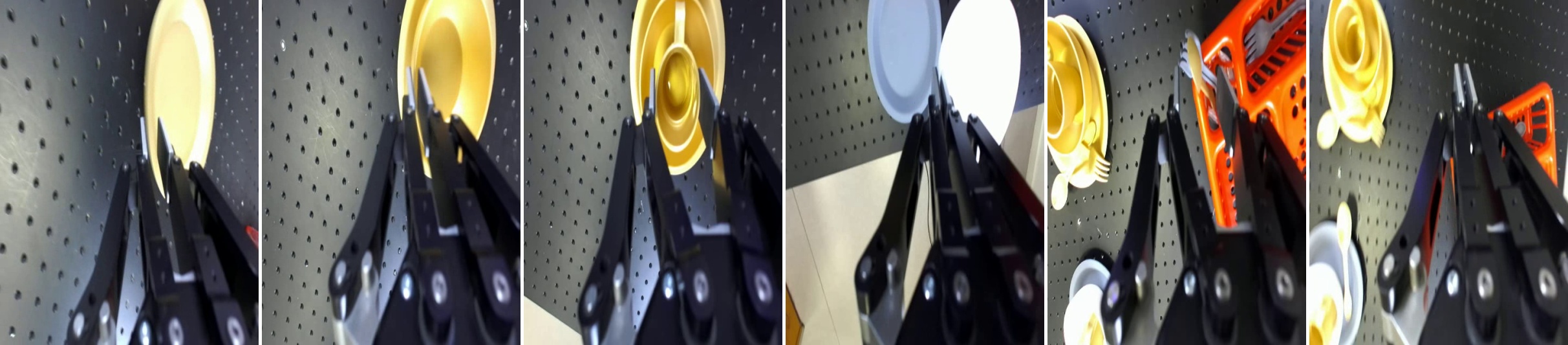}
        \caption{Rollout for the \textit{``Prepare two sets of tableware''} task.}
    \end{subfigure}
    \vspace{-5mm}
\caption{Rollouts of our PriGo-DP on three table-service tasks. PriGo-DP aligns the generated actions with task primitives, thereby improving policy generalization to novel tasks.}
\label{fig:visualization}
\end{figure*}

\subsection{Real-Robot Evaluation}\label{sec:Real-Robot}
\vspace{-3mm}
In our real-world setup, we conduct experiments with a Franka Emika Panda robot and with two RGB sensors. We design three sequential table service in restaurant-style scenarios, including: 1) \textit{prepare a set of tableware}, 2) \textit{bring back a set of tableware to the basket}, 3) \textit{prepare two sets of tableware}. The tableware used in the experiments included a disk, a bowl, a cup, a fork, and a spoon. The averaged success rates over 10 episodes per task are reported in \Cref{tab:realrobot}. Even with no more than 20 demonstration samples for training, our method achieves a success rate of 81\%. In comparison to the vanilla CogACT, our PriGo-DP demonstrates an average performance improvement of 26 absolute points across three tasks. This demonstrates that PriGo is highly beneficial for long sequences.

To illustrate the practical efficacy of our policy, we visualized representative test cases for each of the three tasks. The corresponding qualitative results are presented in \Cref{fig:visualization}, where each subplot depicts the key sequential steps of the task executed from initial state observation to final goal achievement. These visualizations demonstrate the robustness of our approach in real-world tasks.

\subsection{Ablation Studies}\label{sec:ablation}
\vspace{-2mm}
The proposed primitive guidance, PriGo, can operate in a plug-and-play manner—an attribute that has been fully validated by the experimental results presented in 
\Cref{tab:LIBERO,tab:google_robot,tab:calvin,tab:realrobot}. 

\noindent
\begin{minipage}[th!]{0.5\textwidth}
\captionsetup{type=table}
\renewcommand{\arraystretch}{0.8}
\renewcommand{\tabcolsep}{0.7mm}
\caption{Success rate of PriGo-DP (CogACT with PriGo) on three long-horizon real-world tasks. w/o PG and w/ PG mean with and without PriGo.}
\centering
\vspace{-1mm}
\begin{tabular}{@{}l|c|c|c@{}}
\toprule
Task & \#Train & w/o PG & w/ PG \\
\midrule
prepare one tableware& 12 & $47 \pm9.3$& $76\pm7.3$ \\
bring back tableware&  18 & $55 \pm11.3$& $81\pm9.6$\\
prepare two tableware& 20 & $40 \pm8.3$& $62\pm8.0$\\
\bottomrule
\end{tabular}
\label{tab:realrobot}
\end{minipage}
\hfill
\begin{minipage}[th!]{0.48\textwidth}
\captionsetup{type=table}
\renewcommand{\arraystretch}{0.8}
\caption{Ablation study. The policy learned with full 8 primitives outperforms its counterparts that do not distinguish ``push'', ``pull''~(PP), and do not use hybrid primitives~(HP).}
\vspace{-2mm}
\centering
\begin{tabular}{@{}l|c|c|c@{}}
\toprule
Dataset & w/o PP & w/o HP & FULL \\
\midrule
LIBERO & 89.9 & 91.0 & 92.3 \\
SIMPLER & 77.1 & 77.8 & 79.9 \\
\bottomrule
\end{tabular}
\label{tab:ablation}
\end{minipage}


The proposed primitive guidance can operate in a plug-and-play manner—an attribute that has been validated by the experimental results presented in 
\Cref{tab:LIBERO,tab:google_robot,tab:calvin,tab:realrobot}. To further dissect the efficacy of individual components of primitives, we conduct ablation studies on three distinct schemes: 1) Non-discriminated translation primitives (i.e., treating ``push'' and ``pull'' as a unified translational primitive type); 2) None hybrid primitives (excluding ``push+rotation'' and ``pull+rotation''); 3) A full set of eight primitive actions. The experimental results are presented in \Cref{tab:ablation}, which demonstrate that these primitive actions contribute positively to the learning and optimization of policy models.

\begin{figure}[htbp!]
\vspace{-3mm}
\centering
\includegraphics[width=0.49\linewidth]{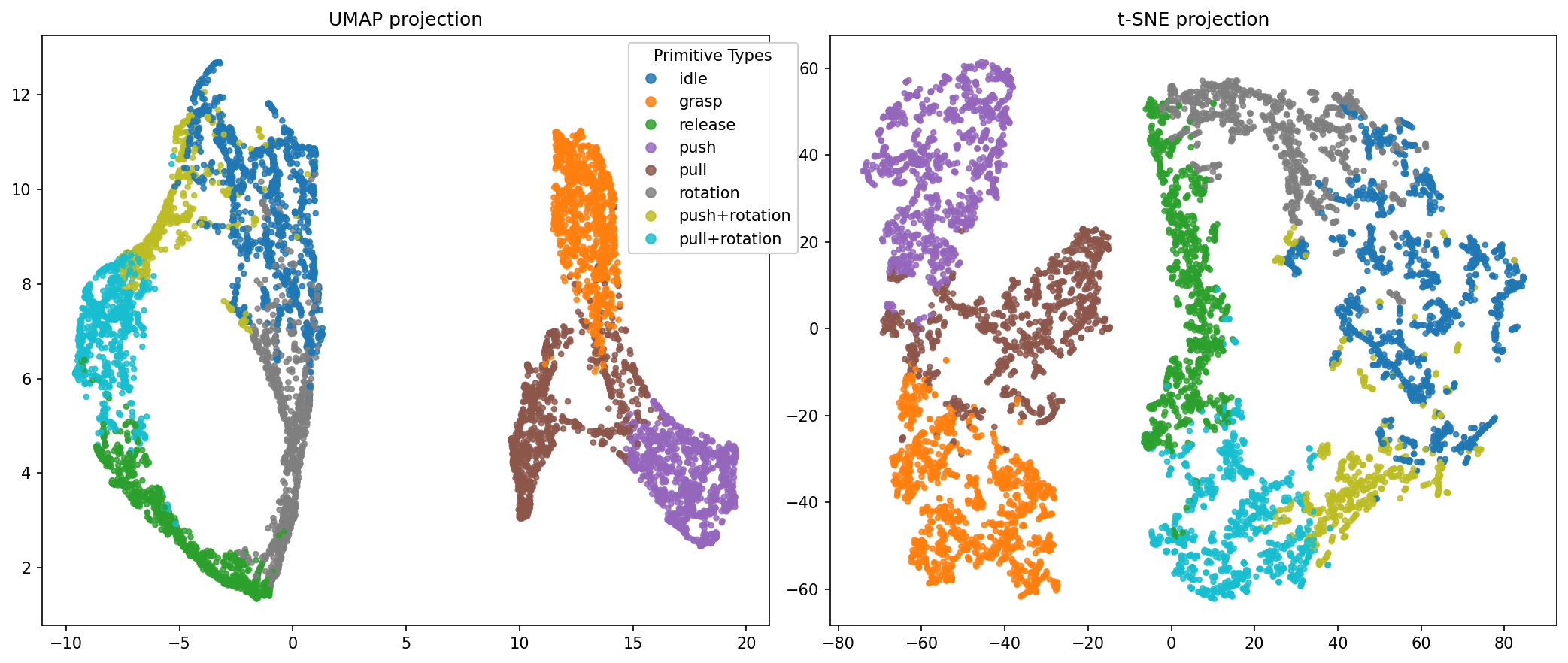}
\includegraphics[width=0.49\linewidth]{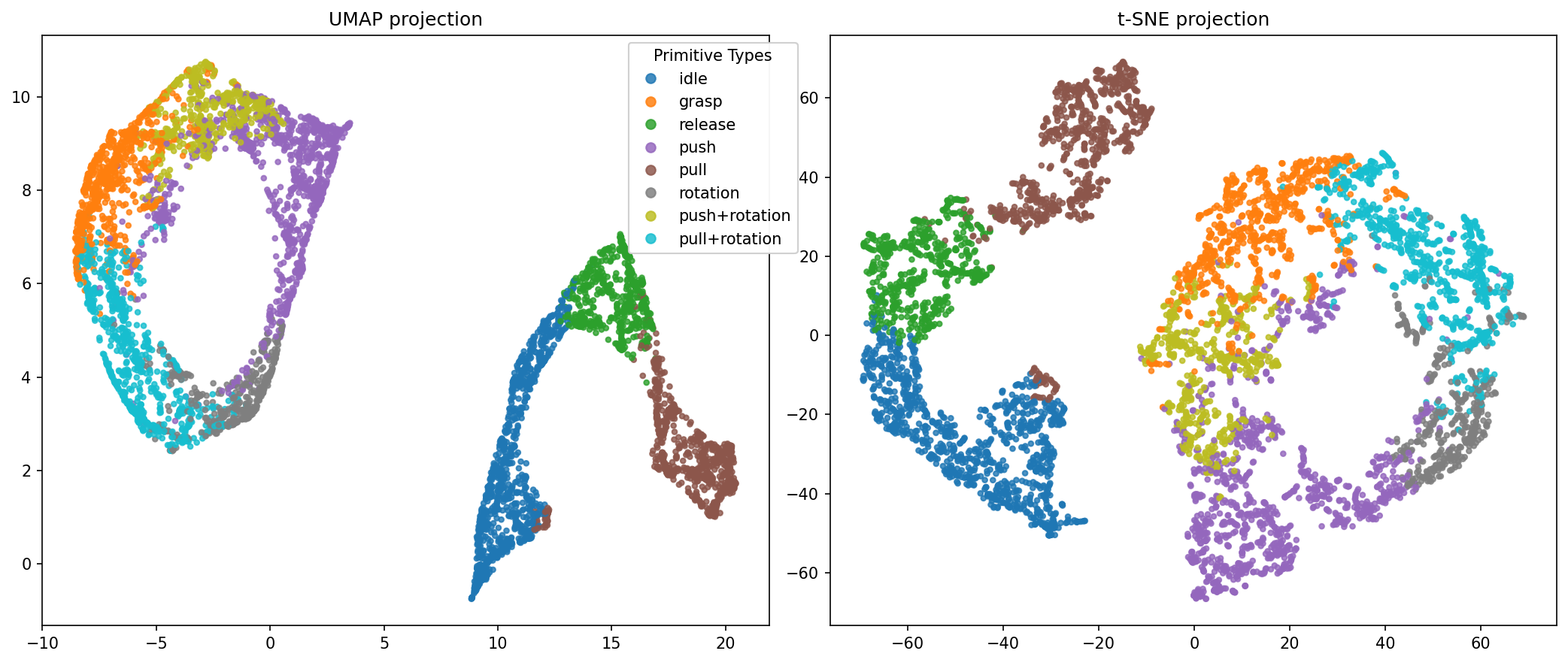}
\vspace{-2mm}
\caption{Visualization of actions and primitive labels from two episodes via UMAP and t-SNE.} 
\label{fig:tsne}
\end{figure}

As shown in \Cref{tab:ablation}, omitting the ``push'' and ``pull''~(PP) types results in the most significant performance degradation of the policy. Without differentiating between push and pull, the generated actions tend to deviate in direction, directly leading to task failure. The model also experiences a performance drop when hybrid primitives are removed. Notably, our observations show that integrating hybrid primitives leads to smoother action outputs from the policy.

Furthermore, to intuitively illustrate the primitive action classification results, we visualize the primitive action instances in \Cref{fig:tsne} using UMAP and t-SNE. The first two subfigures are from the Object subset of LIBERO with 7281 actions, and the third and fourth are from the Spatial subset with 5796 actions. Here, distinct colors denote different categories generated via Eq.~\eqref{eq:HardClassify}. As evidenced by the results, the same primitive exhibits clear clustering after dimensionality reduction.

%% file: sections/5_conclusion.tex
\section{Conclusion}\label{sec:conclusion}
\vspace{-2mm}
In this paper, we introduce PriGo, a primitive guidance approach that can work with both lightweight policies and large VLA models. PriGo effectively decouples discrete primitive planning from continuous action learning. To capture the primitive planning for a given task and its visual observations, we proposed a novel primitive classification module and an automatic primitive labeling method, enabling primitive prediction directly from current observations. To ensure that the generated actions align with the predicted primitives, we introduced a differentiable primitive guidance mechanism, which constrains action generation, leading to more robust manipulation. This guidance can be seamlessly integrated into pretrained diffusion and flow policies, offering a flexible, plug-and-play solution that enhances performance without the need for retraining. Experimental results demonstrate that our policy outperforms baseline methods in both simulation and real-world tasks.

\textbf{Limitations}. Although PriGo improves both flow-based and diffusion-based policies at test time, the gains remain limited. This mainly results from two factors. First, the plug-and-play design enables deployment without retraining but inherently restricts performance improvements. Future work could explore tighter integration with end-to-end training. Second, PANet’s generalization is still constrained by model size and training data scale. More suitable data and incorporating historical robot states may further improve the effectiveness of the guidance mechanism.

%% file: sections/A_appendix.tex
\newpage
\appendix
\onecolumn

This document presents the error-bound proposition, along with implementation details for PANet and a primitive-action analysis. It also provides additional experimental results, including the impact of different methods for deriving discrete primitive action types, and qualitative comparisons on real-world robotic tasks. 

\section{Error Bound} 
\begin{proposition}[Error Bound]
\label{prop:error_contraction}
Assume a pre-trained policy $\epsilon_\theta$ yields the action $\tilde{\as}_t$ at step $t$ with error $\tilde{\mathcal{E}}_t = \|\tilde{\as}_t - \as_{*}\|$ to action $\as_{*}$. When the primitive-guided diffusion denoising or flow velocity field is applied, the resulting action ${\as}_t$ satisfies a strictly tighter error bound:
\begin{equation}\label{eq:error}
{\mathcal{E}}_t \le \frac{1}{1 + C} \cdot \tilde{\mathcal{E}}_t + \mathcal{O}(\tau_t)
\end{equation}
where ${\mathcal{E}}_t = \|{\as}_t - \as_{*}\|$, $C$ is a positive coefficient related to the local curvature of the action manifold and the Lipschitz continuity of the loss function $L$, and $\tau_t$ is the residual term from the stochastic process. 
\end{proposition}

\begin{proof}
Both primitive-guided diffusion denoising and primitive-guided flow updates can be written in a unified form as a guided correction step:
\begin{equation}
\mathbf{a}_t = \tilde{\mathbf{a}}_t - \eta G(\tilde{\mathbf{a}}_t) + \boldsymbol{\xi}_t,
\end{equation}
where $G(\cdot)$ is the primitive-induced guidance field, $\eta>0$ is a step size, and $\boldsymbol{\xi}_t$ is a stochastic residual term satisfying
\begin{equation}
\mathbb{E}\|\boldsymbol{\xi}_t\| = \mathcal{O}(\tau_t).
\end{equation}

We assume that the target action $\mathbf{a}_*$ is a fixed point of the guidance field:
\begin{equation}
G(\mathbf{a}_*) = 0,
\end{equation}
and that $G$ is locally strongly monotone in a neighborhood of $\mathbf{a}_*$: there exists a constant $C_0>0$ such that
\begin{equation}
\langle G(\mathbf{a}) - G(\mathbf{a}_*),\, \mathbf{a} - \mathbf{a}_* \rangle \ge C_0 \|\mathbf{a} - \mathbf{a}_*\|^2.
\end{equation}

Consider the error after the guided update:
\begin{equation}
\mathbf{a}_t - \mathbf{a}_* = (\tilde{\mathbf{a}}_t - \mathbf{a}_*) - \eta G(\tilde{\mathbf{a}}_t) + \boldsymbol{\xi}_t.
\end{equation}

Taking squared norms and expectations yields:
\begin{equation}
\mathbb{E}\|\mathbf{a}_t - \mathbf{a}_*\|^2 =\|\tilde{\mathbf{a}}_t - \mathbf{a}_*\|^2 
- 2 \eta \langle G(\tilde{\mathbf{a}}_t), \tilde{\mathbf{a}}_t - \mathbf{a}_* \rangle
+ \eta^2 \|G(\tilde{\mathbf{a}}_t)\|^2 
+ \mathcal{O}(\tau_t^2).
\end{equation}

Using $G(\mathbf{a}_*)=0$, we have
\begin{equation}
\langle G(\tilde{\mathbf{a}}_t), \tilde{\mathbf{a}}_t - \mathbf{a}_* \rangle
= \langle G(\tilde{\mathbf{a}}_t) - G(\mathbf{a}_*), \tilde{\mathbf{a}}_t - \mathbf{a}_* \rangle
\ge C_0 \|\tilde{\mathbf{a}}_t - \mathbf{a}_*\|^2.
\end{equation}

For sufficiently small $\eta$, the quadratic term $\eta^2 \|G(\tilde{\mathbf{a}}_t)\|^2$ can be absorbed into the contraction coefficient, yielding
\begin{equation}
\mathbb{E}\|\mathbf{a}_t - \mathbf{a}_*\|^2 \le (1 - 2 \eta C_0) \|\tilde{\mathbf{a}}_t - \mathbf{a}_*\|^2 + \mathcal{O}(\tau_t^2).
\end{equation}

Taking square roots and choosing $C$ such that
\begin{equation}
\frac{1}{1 + C} = \sqrt{1 - 2 \eta C_0} ,
\end{equation}
we arrive at the final bound:
\begin{equation}
\|\mathbf{a}_t - \mathbf{a}_*\| \le \frac{1}{1 + C} \|\tilde{\mathbf{a}}_t - \mathbf{a}_*\| + \mathcal{O}(\tau_t).
\end{equation}

Since $C>0$, the factor $1 / (1 + C) < 1$, which guarantees accelerated convergence toward the target action $\mathbf{a}_*$.  
\end{proof}

\section{Details Implementation and Primitives Analysis} \label{sec:Implementation}
This section elaborates on additional details regarding the implementation and training of PANet, the primitive classifier proposed in our work.

PANet takes the observation $\ob$ as inputs and outputs the primitive type of the next action, i.e., $y^k=\text{PANet}(\ob^k)$. For implementation, we adopt T5~\cite{2020t5} as the text encoder and DINOv2~\cite{oquab2024dinov2} as the image encoder. Subsequently, we utilize the multi-modal causal Q-Former~\cite{li2023blip} to fuse features from these two modalities, which are then fed into the classification head for class probability output. The Q-Former is configured with 3 layers, an embedding dimension of 1280, a cross-attention dimension of 768, 8 attention heads, and 32 queries. The classification head employs a multi-layer perceptron (MLP) architecture consisting of three MLP layers interleaved with ReLU activation functions (i.e., MLP → ReLU → MLP → ReLU → MLP). The input and output feature dimensions of the classification head are 1280 and 8, respectively: 1280 corresponds to the concatenated dimension of text and image features, while 8 denotes the number of primitive action types. The middle MLP layer is configured with 256 input and output channels. For detailed implementation details, refer to our publicly available codebase at \href{https://github.com/OTCVCG/PriGo}{\textcolor{myblue}{\textbf{PriGo}}}.

For training, PANet is optimized using cross-entropy loss, where the ground-truth labels are obtained via Eq.~(2) in our manuscript. We employ the AdamW optimizer with hyperparameters $\beta_1=0.9$, $\beta_2=0.99$, and a learning rate of $lr=5\times10^{-5}$. 
The model is trained on LIBERO using a single NVIDIA A100 GPU, with a total training time of approximately 1 day.

\textbf{Statistical analysis of primitive actions.}
We conducted preliminary experiments to analyze the distribution of primitive types under different settings, and selected a set of eight primitive actions for all experiments based on these statistics. Specifically, we observed that certain composite primitives are extremely rare, e.g., \textit{grasp + push} accounting for only 0.07\% across all 10 categories. This leads to a long-tail classification problem, degrading both overall accuracy and performance on infrequent classes. Moreover, over-specifying primitives can introduce ambiguity. For instance, splitting rotation into left and right directions is often unnecessary, as both may be valid in many scenarios (e.g., grasping a handle or a cup). Enforcing a fixed label in such cases reduces policy flexibility and introduces label noise. In addition, finer-grained primitives result in greater variation in primitive distributions across tasks, which weakens cross-task generalization. 

For example, when rotation is split into clockwise and counterclockwise directions, the primitive distribution for the task ``\textit{pick up the salad dressing and place it in the basket}'' becomes: \{`0': 14.43\%, `1': 5.05\%, `2': 4.48\%, `3': 8.73\%, `4': 12.85\%, `5': 10.79\%, `6': 2.80\%, `7': 14.51\%, `8': 6.95\%, `9': 12.07\%, `10': 7.34\%\}. For “\textit{pick up the cream cheese and place it in the basket}”, it becomes: \{`0': 13.27\%, `1': 5.14\%, `2': 7.72\%, `3': 10.86\%, `4': 10.86\%, `5': 5.79\%, `6': 11.60\%, `7': 4.23\%, `8': 9.10\%, `9': 6.64\%, `10': 14.79\%\}.  
Notably, the distribution of the last categories changes significantly across tasks.  In contrast, using the original 8 primitives yields more stable distributions. For the first task: {`0': 16.04\%, `1': 6.34\%, `2': 5.16\%, `3': 8.47\%, `4': 13.58\%, `5': 13.94\%, `6': 18.98\%, `7': 17.49\%},  and for the second task: \{`0': 13.34\%, `1': 7.18\%, `2': 7.76\%, `3': 10.91\%, `4': 10.91\%, `5': 16.87\%, `6': 14.30\%, `7': 18.73\%\}.

\section{Additional Results} \label{sec:result}
In this section, we further evaluate and visualize the effectiveness of the proposed primitive guidance, PriGo, on robotic manipulation tasks.

Two optional strategies are available to get discrete class output: selecting the maximum-probability index or conducting probability-based sampling. We conducted experiments on LIBERO to explore the impact of these two methods on subsequent robot manipulation, with the results presented in Table~\ref{tab:LIBERO_SUPP}. While sampling yields a mere 0.2\% performance improvement, this gain is negligible in practice. Given such a limited effect alongside the additional computational overhead introduced by sampling, we instead adopt the simpler and more direct maximum-probability index strategy for our implementation.

To evaluate inference overhead on real robot tasks, we perform a wall-clock time analysis comparing PriGo with the unguided baseline. On a single NVIDIA RTX A6000 GPU, the baseline (CogACT) requires approximately 181 ms per step, while PriGo increases the latency to 208 ms per step, corresponding to a 15\% overhead. However, since each inference cycle generates a sequence of $k=16$ actions, this yields an effective control frequency of 77 Hz, which exceeds the standard real-time requirement for reactive manipulation (typically around 50 Hz). Furthermore, we measure the average task completion time and observe a 10\% reduction with our method (2.56 minutes) compared to the baseline (2.85 minutes). This improvement stems from the enhanced structural coherence of the generated action sequences, which enables smoother subtask execution and reduces the total number of steps required to reach the goal.

\begin{table}[htb!]
\caption{Success rate on LIBERO with four settings (spatial, object, goal, long-horizon). Comparison between maximum-probability index (argmax) and probability-based sampling (sampling).}
\centering
\begin{tabular}{l|c|c|c|c|c}
\toprule
Methods& Spatial  & Object&  Goal &	Long & \cellcolor{ImportantColor}Avg. Succ.\\
\midrule
argmax & 95.6 &97.3 & 95.8 &  80.4     &\cellcolor{ImportantColor} 92.3 \\
sampling & 95.2 & 97.6 &96.3 &  80.9   &\cellcolor{ImportantColor} 92.5 \\
\bottomrule
\end{tabular}
\label{tab:LIBERO_SUPP}
\end{table}

To illustrate the corrective effect of our guidance, we selected five representative frames from the policy rollouts with and without incorporating PriGo for the same task. These results are visualized in Figure~\ref{fig:visulization}, where the top row displays the outcomes of our PriGo-DP and the bottom row shows the baseline diffusion policy CogACT. From the results, it can be observed that our method produces action sequences that align well with the required primitive types: rotation (placing the plate), push and release (placing the bowl onto the plate), and grasping the handle followed by release (placing the cup into the bowl). In contrast, the baseline fails to generate structured action sequences. Specifically, it lacks the required rotation (resulting in an inverted plate), introduces unnecessary rotations (during bowl placement), and exhibits incorrect grasping behavior (grasping the inner wall of the cup) along with premature release in mid-air (during cup placement).

\begin{figure}[htbp!]
\vspace{-1mm}
\centering
\includegraphics[width=1.0\linewidth]{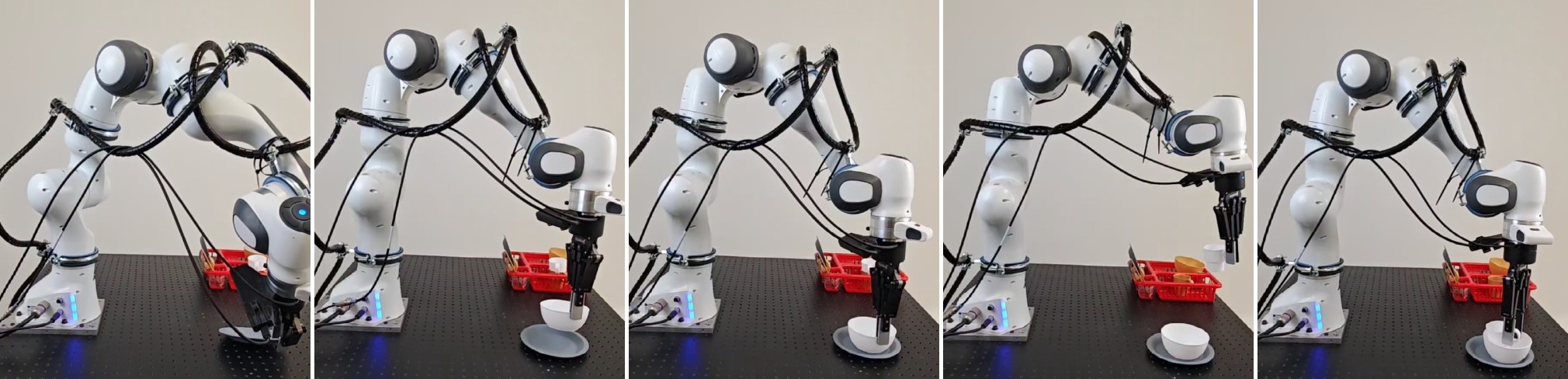}\\
\includegraphics[width=1.0\linewidth]{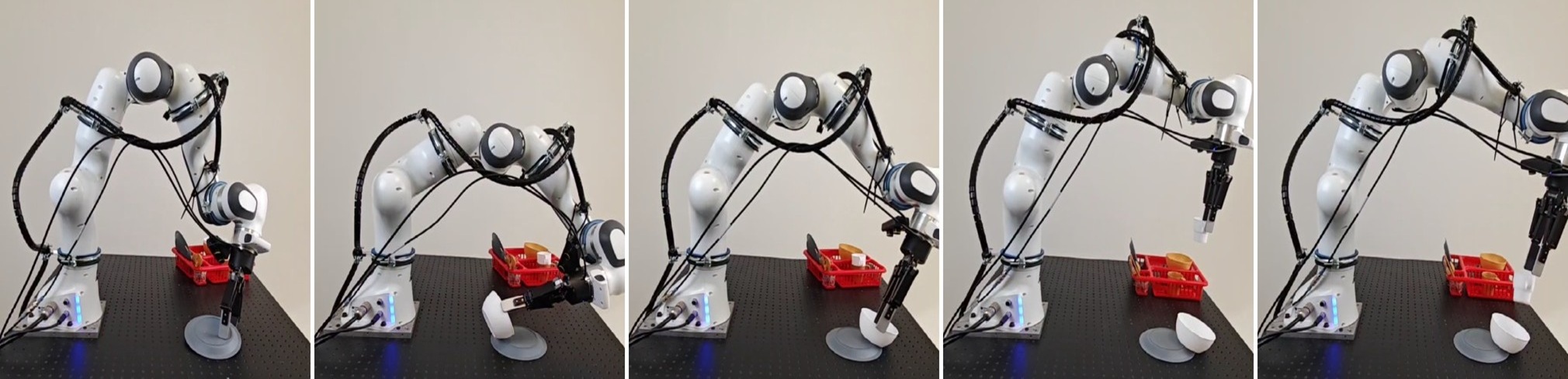}
\caption{Comparison of our PriGo-DP (top row) and the vanilla diffusion policy (bottom row) on the `prepare a set of tableware' task.} 
\label{fig:visulization}
\end{figure}